\documentclass{llncs}

\usepackage{natbib}

\usepackage[pdftex]{graphicx}
\usepackage[ruled, linesnumbered]{algorithm2e}
\usepackage{multirow}
\usepackage{multirow} 
\usepackage{booktabs} 
\usepackage[cmex10]{amsmath}
\usepackage{amssymb}
\usepackage{url}
\usepackage{color}
\usepackage{array}
\usepackage[caption=false,font=footnotesize,labelfont=sf,textfont=sf]{subfig}
\usepackage{fixltx2e}

\usepackage{fancyhdr}
\usepackage{lastpage}

\usepackage{units}

\graphicspath{{figures/}}

\newcommand{\NP}{\ensuremath{\mathbf{NP}}}
\newcommand{\NPhard}{\NP-hard\xspace}

\newcommand{\algcoreclustering}{1\xspace}
\newcommand{\algcoocpair}{2\xspace}
\newcommand{\algcoocbs}{3\xspace}

\begin{document}

\title{Clustering with Confidence:\\Finding Clusters with Statistical Guarantees}

\author{Andreas Henelius\inst{1},
        Kai Puolam\"{a}ki\inst{1},
        Henrik Bostr\"om\inst{2},
        Panagiotis Papapetrou\inst{2}}
\institute{Finnish Institute of Occupational Health, PO Box 40, FI-00251 Helsinki, Finland. E-mail: \{andreas.henelius, kai.puolamaki\}@ttl.fi
  \and
Department of Computer and Systems Sciences, Stockholm University, PO Box 7003, SE-164 07 Kista, Sweden. E-mail: \{henrik.bostrom, panagiotis\}@dsv.su.se}%

\maketitle

\begin{abstract}
Clustering is a widely used unsupervised learning method for finding structure in the data. However, the resulting clusters are typically presented without any guarantees on their robustness; slightly changing the used data sample or re-running a clustering algorithm involving some stochastic component may lead to completely different clusters. There is, hence, a need for techniques that can quantify the instability of the generated clusters. In this study, we propose a technique for quantifying the instability of a clustering solution and for finding robust clusters, termed \emph{core clusters}, which correspond to clusters where the co-occurrence probability of each data item within a cluster is at least $1 - \alpha$.  We demonstrate how solving the core clustering problem is linked to finding the largest maximal cliques in a graph. We show that the method can be used with both clustering and classification algorithms. The proposed method is tested on both simulated and real datasets. The results show that the obtained clusters indeed meet the guarantees on robustness.
\end{abstract}

\begin{keywords}
Clustering
\end{keywords}

\setcounter{footnote}{0}

\section{Introduction}
Clustering is a fundamental and widely used unsupervised learning
technique for extracting and understanding structural properties of
datasets \citep{hartigan75}. The input to a clustering algorithm is a
set of objects, often represented in the $d$-dimensional space
$\mathbb{R}^d$ with a distance or similarity measure $D$. Clustering
can also be performed in more abstract spaces, e.g., clustering of
strings or graphs. The objective is to assign objects into groups, so
that similar objects are placed in the same group and dissimilar
objects are placed in different groups. The potential of clustering
algorithms to reveal the underlying structure in any given dataset has
been exploited in a wide variety of domains, such as image processing,
bioinformatics, geoscience, and retail marketing.

The quality of clustering is affected by many factors, such as
noise in the data or the suitability of the assumptions of the
clustering algorithm. Such factors can, in general, prevent any
notion of stability in the clustering result. In particular, two
data items that are assigned to the same cluster by a clustering
algorithm may end up in different clusters if the algorithm is
re-run. This variation in the clustering of a dataset can be
seen as stemming from both systematic and random errors. The
systematic error is due to the stochastic nature of the used
clustering algorithm, i.e., re-clustering of dataset using some
algorithms leads to slightly different solutions due to, e.g.,
different initial conditions. The random error in the clustering
is due to the variation in the used data sample, i.e., choosing a
slightly different data sample will likely change the clustering
solution too. These two sources of error in the clustering
solution are not easily separable. Nonetheless, knowledge of the
stability of the clustering result is important for the proper
interpretation of the structure of the data.

In this paper, we introduce a method for assessing the stability of a
clustering result in terms of \emph{co-occurrence probability}, i.e.,
the probability that data items co-occur in the same
cluster. Specifically, we address the following problem: given a
dataset and a clustering algorithm, we want to identify the largest
set of items within each cluster in the dataset, so that the
co-occurrence probability of the items in these within-cluster sets is
guaranteed to be at least $1-\alpha$, where $\alpha \in [0,1]$. Each
such set is called a \emph{core cluster} and the addressed problem is
referred to as the \emph{core clustering problem}. The core clustering
method hence reflects the combined effect of the systematic and random
errors in the clustering of a dataset.

\smallskip
\noindent
\textbf{Example.} To better motivate the idea of core clustering,
consider the following example, using a synthetic dataset of $n$
points generated by a mixture of three Gaussians with unit
variance. Suppose that this dataset represents data from patients,
each suffering from one out of three possible conditions, and that the
condition of individual patients is unknown by a clinician studying
the data.

The clinician is interested in grouping the patients in order to
investigate patterns in the data, such as which patients are similar
to each other and thus more likely to have the same disease. The
clinician employs the k-means++ algorithm \citep{vassilvitskii:2007:a}
to partition the data. However, although clustering the data yields a
result, it is difficult to interpret the validity of the result, i.e.,
how good the clustering result is. Do the patients in the clusters
really belong to these groups?

Our goal with the proposed core clustering algorithm is to answer this
question by providing a stable clustering result that also gives a
statistical guarantee that the data points (in our example patients)
in a given cluster co-occur with a probability of at least $1 -
\alpha$, where $\alpha \in [0,1]$ is a given constant.

One way of evaluating the validity of the clustering solution in this
example could be by choosing a slightly different sample of patients
from the same population. Repeatedly clustering such samples would
allow us to track how often a certain patient is placed in a
particular group. The intuition is that if some patients are often
placed together in the same particular group, it indicates that these
patients are strong members of that group. In contrast, it is
difficult to label patients that shift from one group to another
during different clusterings.

We will now extend this idea of clustering slightly overlapping
samples of data into a method that allows us to reach our goal to
provide a statistical guarantee on the co-occurrence of data points in
clusters. To this end, we employ the following scheme: First, we run
k-means++ and record the original clustering. Next, we take a
bootstrap sample from the original dataset and re-run k-means++ for
that sample. This step is repeated, e.g., $1000$ times, each time
using a new bootstrap sample. The co-occurrence probability of each
pair of points is determined as the fraction of times the points have
co-occurred in the same cluster during the process. Hence, we can now
identify the stable \emph{core clusters} as the largest set of points
within each original cluster, where the points co-occur with a
probability of at least $1-\alpha$. The number of core clusters is
equal to the number of original clusters, and the method of core
clustering can be seen as refining the output from a clustering
algorithm, such as k-means++ used here, finding points within clusters
having a strong cluster membership and by excluding certain unstable
points from the clustering solution. The clustering of the dataset
discussed in our example is shown in Figure~\ref{fig:ex:coreclusters},
where the clusters obtained from the original clustering are shown
using different plot symbols (circles, triangles, and rectangles). The
points belonging to the core clusters are shown as filled points, and
these points always co-occur in the same cluster with a probability of
at least $1 - \alpha$. The unfilled points in the figure do not belong
to the core clusters and are termed \emph{weak points}, since their
cluster membership is weak as it changes from one clustering to
another in more than $\alpha$ percent of the cases.

The fact that the points in the core clusters are guaranteed to
co-occur with a certain probability is useful in exploratory data
mining, and in the above example they would allow the clinician to
cluster patients using different confidence levels. This would
correspond to testing the hypothesis that two data items belong to the
same cluster with a given probability.

The core clusters also allow the behaviour of the clustering algorithm
to be evaluated. If the core clusters are very small, it means that
the clustering is unstable and the results are unreliable. Core
clustering hence offers direct insight into the suitability of a
clustering scheme, reflecting the interaction between the clustering
algorithm, its settings and the data.

Since clustering is a fundamental tool in exploratory data analysis,
the concept of core clustering could be useful in domains typically
using clustering, e.g., in bioinformatics.

\smallskip
\noindent
\textbf{Related work.} Next, we summarise some related work and
position core clustering with respect to existing
literature. Clustering can be performed in a variety of ways and a
plethora of clustering methods has been proposed. However, a thorough
review of clustering methods is outside the scope of this paper. The
reader can refer to several available surveys (\cite{SurveyKogan,
  Jain1999, Guinepain2005}). Clustering can be either \emph{hard},
which is the focus of this work, where each data object strictly
belongs to a single cluster, or \emph{soft}, where each data object
may belong to more than one partition. Some widely used algorithms for
hard clustering include k-means and variants of it
\citep{James2013,Pelleg2000}, hierarchical clustering, and
density-based clustering. For the case of soft clustering, also known
as fuzzy clustering \citep{fuzzysurvey}, several approaches exist in
the literature, including standard algorithms such as fuzzy c-means
and extensions \citep{Bezdek1981, hoppner99,Wu03} and the
Gustafson-Kessel algorithm \citep{Gustafson79}. Other adaptations of
these algorithms include the probabilistic c-means
\citep{Krishnapuram1993} and the probabilistic fuzzy c-means
\citep{Tushir2010} algorithms, which are motivated by the fact that
both probabilistic and membership coefficients are necessary to
perform clustering, respectively to reduce outlier sensitivity and to
assign data objects to clusters. Thus, both relative and absolute
resemblance to cluster centres are taken into account. Similar
probabilistic clustering approaches are based on the idea that the
underlying data objects are generated by $k$ different probability
distributions
\citep{Smyth96clusteringusing, mclachlan98, Taskar2001}. The $k$ models
correspond to the clusters of interest and the task is to discover
these models. However, none of these clustering algorithms can provide
statistical guarantees on the result.

This work can be compared with previous work on assessing the significance
of clustering solutions in \citet{Ojala2010ICDM} and
\citet{Ojala2011}, which however tries to answer the question whether
the clustering solution as a whole is statistically valid by
constructing elaborate null models for the data, as opposed to
summarising individual co-occurrence patterns as done in this paper.

In the field of robust clustering, the goal is to perform clustering
without the result being overly affected by outliers in the data. See,
e.g., \cite{garcia:2010:a} for an overview. The problem studied in
this paper is different from robust clustering, since in our case the
focus is not on removing outlier data points. In contrast, unstable
weak points are identified, which are typically not outliers but
instead they are points located on the border of different
clusters. Hence, the two problems are rather complementary, since core
clustering can be used in conjunction with robust clustering methods,
as demonstrated in this paper.

Another important problem is the evaluation of clustering stability,
which is a model selection problem that focuses on discovering the
number of clusters in the data; for a review see
\citet{Luxburg:2010:a}. In this paper, we do not consider clustering
stability with respect to different choices of the number of
clusters. Instead, we study the problem of finding clusters that are
stable upon clustering of resampled versions of the original dataset.

Pairwise comparison of the cluster membership of different points is a
natural and popular method used extensively when comparing the
similarity of two clustering solutions (e.g., \citet{rand:1971:a,
  fowlkes:1983:a}). For reviews of similarity measures related to the
pairwise comparison of clusterings see, e.g., \citet{wagner:2007:a,
  pfitzner:2009:a}.

Core clustering is a variant of clustering aggregation (consensus
clustering), see, e.g., \citet{ghosh:2013:a}. Given a set of different
clustering results for some data, the goal of clustering aggregation
is to find the clustering with the highest mutual agreement between
all of the clusterings in the set. Clustering aggregation can be
performed using for instance probabilistic models (e.g.,
\citet{wang:2011:a, topchy:2005:a}), methods based on on pairwise
similarity (e.g., \citet{gionis:2007:a, nguyen:2007:a}), or methods
considering the pairwise co-occurrences of points (e.g.,
\citet{fred:2005:a}).

Methods in which the co-occurrences of points are used to find stable
clustering solutions have been presented in fields such as
neuroinformatics \citep{bellec:2010:a}, bioinformatics
\citep{kellam:2001:a} and statistics \citep{steinley:2008:a}.

The methods for finding stable clusters by considering co-occurrences
are essentially variations of the evidence accumulation algorithm
proposed by \citet{fred:2001:a, fred:2002:a} and
\citet{fred:2005:a}. In this method clusters are formed by clustering
the co-occurrence matrix of items obtained from different clusterings
of the data.

The method of core clustering also belongs to this category of methods
and the core clusters are formed such that the agreement between all
clusterings is at least $1-\alpha$. Bootstrapping is used to form the
co-occurrence matrix as also done by \citet{bellec:2010:a}. In the
present paper we show the utility of the core clustering method for
explorative data analysis using several clustering and classification
algorithms.


\begin{figure}[!ht]
\centering
    \includegraphics[width = \columnwidth]{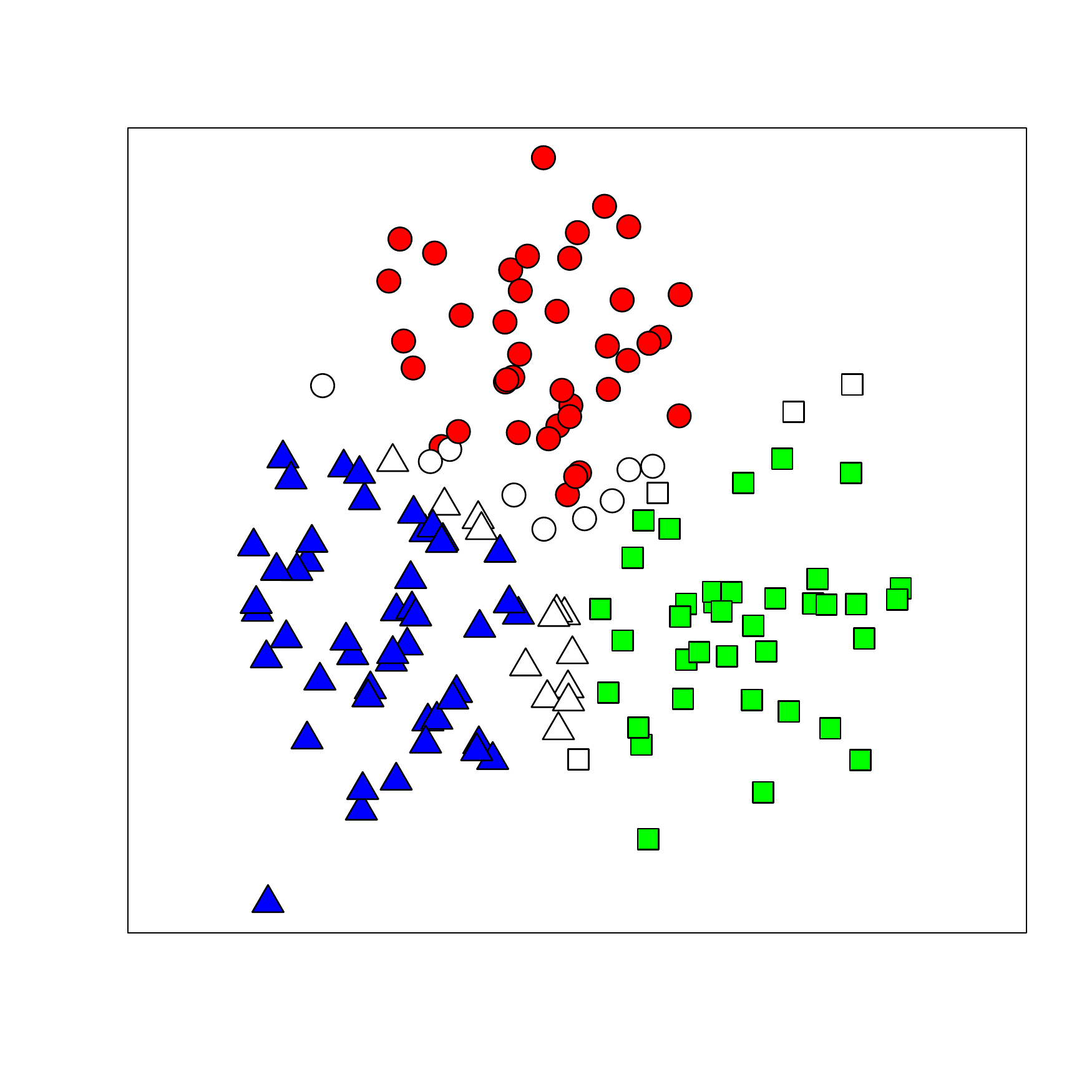}
    \caption{Core clustering of the synthetic dataset using k-means++,
      computed with Algorithm \algcoocbs (Fig.~\ref{alg:cooc:bs})
      using 1000 iterations and $\alpha=0.1$. The original clustering
      is shown using different symbols for each cluster. Core clusters
      are shown with filled symbols and weak points are unfilled.}
    \label{fig:ex:coreclusters}
\end{figure}

\smallskip
\noindent
\textbf{Contributions.} The main contributions of this paper can be
summarised as follows:
\begin{itemize}
\item We present a method for assessing the stability of clusters in
  terms of {\it core clusters} and we define and analyse theoretically
  the corresponding {\it core clustering problem}.
\item An algorithm for solving the core clustering problem is
  proposed.
\item It is demonstrated that bootstrapping provides a feasible method for evaluating clustering stability.
\item Through an empirical evaluation on both synthetic and real-world
  datasets, and for both clustering and classification problems, the
  proposed algorithm is shown to support the interpretation of the
  structure of the datasets.
\end{itemize}

In the next section, we formalise the core clustering problem, and in
Sec.~\ref{methods} the core clustering algorithm is described and
analysed theoretically. In Sec.~\ref{experiments}, the setup and
results of the empirical investigation are presented. The results are
discussed in Sec.~\ref{discussion}, while the conclusions are
summarised and pointers to future work are provided in Sec.~\ref{cr}.

\section{Problem Setting}
\label{problem}

\subsection{Definitions}
Let $D \in X^n$ be a dataset of $n$ items defined in some space
$X$. We denote the $i$th data item in $D$ as $D[i]$, and assume that
these data items have been drawn i.i.d. from an unknown distribution
$F$ over $X$.

We now consider the problem of clustering the data items in $D$. In
other words, our aim is to assign a cluster label to each item in
$D$. In a general case, we can express the process of assigning
cluster (or class) labels to data items using a \emph{clustering
  function}.

\begin{definition}\emph{Clustering function.}
  \label{def:clust:functional}
Given a dataset of $n$ items $D\in X^n$, a clustering function $f_D :
X \mapsto \mathbb{N}$ is a partial function on $X$, having at least
$D$ in its support, which outputs a cluster index for a data item $x$
in its support; the cluster index is denoted as $f_D(x)$.
\end{definition}
In the case of the k-means++ algorithm, the function $f_D$ simply
assigns each item with the index of the closest cluster centroid. For
some clustering algorithms $f_D$ is defined only for items in $D$,
since there may be no natural way to assign a previously unseen data
item into a cluster.

A clustering function can also be used in a supervised setting, where
each data item in $D$ is associated with a class label, and we can use
$D$ to train a classifier function, such as a random forest or a
support vector machine (SVM). In this case, the classifiers predict
the class label given a data item, and hence, we can view the
classifier function as a clustering function $f_D$. As a result, even
though there are differences in terms of operation between an
unsupervised clustering algorithm and a supervised classifier, we can
treat both cases as an assignment task, where each data item in $X$ is
assigned with either a cluster index or a class label.


We can now define the \emph{co-occurrence probability} for two data
items as follows:
\begin{definition}\emph{Co-occurrence probability.}
\label{def:cooc}
Given two data items $x$ and $y$ in $D \subseteq X$, the
\emph{co-occurrence probability} of $x$ and $y$ is the probability
that they co-occur in the same cluster for a data set consisting of
$x,y$ as well as $n-2$ data items drawn at random from $X$ according
to distribution $F$.
\end{definition}
The co-occurrence probability is affected both the systematic and
random error in the clustering due to the randomness in the clustering
algorithm and the variation in the data, as discussed above, and these
effects cannot readily be separated.

Next, we proceed to define the \emph{core clusters}. Clustering the
dataset $D$ into $k$ clusters using $f_D$ gives a disjoint partition
of $D$ into $k$ distinct sets (clusters) $S_i$: $D = \bigcup_i S_i, \,
S_i \bigcap S_j = \emptyset, \, i\neq j$. The set of items in the
clusters $S_i$ can now be refined as follows based on the
co-occurrence probabilities to give the core clusters:

\begin{definition}\emph{Core cluster of $S$.}
\label{def:corecluster}
  Given a set of items $S$ and a constant $\alpha \in \left[0,
    1\right]$, the subset $C \subseteq S$ is a \emph{core cluster} of
  $S$, if $C$ is the largest set of items in $S$, where the
  co-occurrence probability of each pair of items $x, y \in C$ is at
  least $1 - \alpha$.
\end{definition}
In other words, we require that any two points that belong to the same
core cluster co-occur in that cluster with a probability of at least
$1 - \alpha$.

The determination of which data items that belong to the core clusters
hence depends only on the co-occurrence probabilities. Given the
co-occurrence probabilities of data items we can now proceed as
follows to obtain the core clusters. Let $\{S_1, \ldots, S_k\}$ be the
original partition of dataset $D$ into a set of $k$ clusters. The
co-occurrences between all pairs of data items in each cluster $S_i$
can be represented by a graph, where an edge is defined between items
$i$ and $j$, if and only if, their co-occurrence probability in $S_i$
is at least $1 - \alpha$. Within each cluster $S_i$, we proceed with
finding the largest maximal clique $C_i \subseteq S_i$, resulting in
$k$ cliques. For these cliques, every two distinct data items are
adjacent, and the condition of Definition~\ref{def:corecluster} holds;
the $k$ core clusters for $D$ are given by the $k$ largest maximal
cliques $C_i$.

The core clusters, hence, consist of those points in the ``cores'' of
the original clusters, for which the cluster membership is
strong. Core clustering can be thought of as a method of refining the
original clustering, while the interpretation of the clusters is
unmodified.

We also define \emph{weak points} as follows.
\begin{definition}\emph{Weak points.}
  Given a dataset $D$ with $k$ core clusters $\{C_1,\ldots, C_k$\}, the
  set $W$ of weak points is $W = D \setminus \bigcup_i C_i$.
\end{definition}
The weak points are hence those data items that do not belong to the core clusters.

Finally, using the above definitions, we can formulate the main problem studied in this paper.
\begin{problem}\label{prob:1} \emph{Core clustering.}
Given a dataset $D$ of $n$ items, a clustering function $f_D$, and a
constant $\alpha \in \left[0, 1\right]$, we want to find the largest
sets of items within each cluster of $D$ given by $f_D(D)$, where all
pairs of items co-occur with a probability of at least $1 - \alpha$ in
the same cluster.
\end{problem}

Our main objective is therefore, given a dataset and a result of a
clustering algorithm, to refine or enhance this result by identifying
those data points within the clusters for which we can provide
probabilistic guarantees that the data items in the same cluster
co-occur with a probability of at least $1 - \alpha$.

\section{Methods}
\label{methods}
In this section, we present our algorithm for finding core clusters
and prove that the core clustering problem is \NPhard.

\subsection{Algorithm for Determining Core Clusters}
As noted above, the determination of core clusters only depends on
calculating the co-occurrence probabilities of the data items in the
dataset using a given clustering function.

Depending on the knowledge of the distribution $F$ of the data points
in the dataset, we can different strategies for calculating the
co-occurrence probability. If the distribution $F$ of the data items
is known, we can use this distribution directly. However, in most
cases the true distribution is unknown and we resort to bootstrapping
the original dataset. Below, we present algorithms for both of these
scenarios.

The general algorithm for determining core clusters is presented in
detail in Algorithm \algcoreclustering
(Fig.~\ref{alg:coreclustering}). The main steps of the algorithm are
as follows:

 \begin{enumerate}
\item Given a dataset $D$, determine an initial clustering using a
  suitable clustering function (line 1 in
  Fig.~\ref{alg:coreclustering}). This could, for example, correspond
  to running k-means++ on $D$. The role of the initial clustering is
  discussed below in more detail.
\item Calculate the co-occurrence probabilities for all pairs of data
  items in each of the clusters (line 5). This step can be achieved in
  multiple ways, and below we present two methods.
\item The process of finding the core clusters as the largest set of
  items within the original cluster where pairs of items co-occur at a
  given level $1-\alpha$ naturally leads to a clique-finding problem
  when we extend the pairwise co-occurrences of items to comprise the
  set of all items co-occurring at the given level. Hence, as
  discussed above, to find the core clusters of $D$ we find the
  largest maximal cliques within each of the original clusters (line
  7). There are many choices for an algorithm to find the largest
  cliques, see e.g., \cite{bron:1973:a}.
\end{enumerate}

\begin{figure}[!t]
\begin{algorithm}[H]
\DontPrintSemicolon
\SetKwFunction{CoreClusters}{CoreClustering}
\SetKwFunction{ClusterFunction}{$f_D$}
\SetKwFunction{CoOccurrenceFunction}{CoOccurrenceProbabilities}
\SetKwFunction{DiscretizeFunction}{Discretize}
\SetKwFunction{FLMC}{FindLargestMaximalClique}
\SetKwInOut{Input}{input}
\SetKwInOut{Output}{output}
\Input{
\begin{itemize}\itemsep0em
\item[--] dataset $D$ with $n$ data items,\\
\item[--] confidence threshold $\alpha$,\\
\item[--] a clustering function $f_D$,\\
\item[--] a function \CoOccurrenceFunction to calculate the co-occurrence probabilities of data items,\\
\item[--] \FLMC a function that finds the largest maximal clique in a given graph
\end{itemize}
} \Output{a set of core clusters} \BlankLine $K \leftarrow f_D(D)$
\tcc{Initial clustering. $K$ is an $n$-dimensional vector with cluster
  indices} Let $I$ be a set of unique indices in $K$\; $C \leftarrow
\emptyset$ \; \For{$i$ {\rm in} $I$}{ \tcc{For the subgraph consisting
    of points belonging to the original cluster $i$, determine the
    largest maximal clique} Compute the co-occurrence probabilities
  between all pairs of data items with cluster index $i$ using
  \CoOccurrenceFunction\; Let $G$ be an undirected graph of nodes
  having a cluster index $i$ such that there exist an edge between two
  nodes iff the co-occurrence probability is at least $1-\alpha$\; $C
  \leftarrow C \cup \{\FLMC \left(G\right)\}$ }

\Return{C} \tcc{Return the set of core clusters}
\caption{CoreClustering}
\end{algorithm}
\caption{ Algorithm for finding the core clusters of a dataset.}
\label{alg:coreclustering}
\end{figure}


Next, we present two algorithms for determining the co-occurrence
probability (step 2 of Algorithm \algcoreclustering); one direct,
na\"ive approach and one based on bootstrapping. For both algorithms
we also derive the time complexity of the algorithms and the required
number of samples needed to reach a sufficient co-occurrence
probability accuracy.

\subsubsection{Sampling $F$}
\label{sec:methods:F}
If the data distribution $F$ is known, e.g., as in the example
presented in the Introduction, the co-occurrence probabilities can be
computed to the desired accuracy using Algorithm \algcoocpair
(Fig.~\ref{alg:cooc:pair}) by sampling sets of $n-2$ data points from
$F$ and applying the clustering function $f_D$ to the set consisting
of $x$, $y$, and the $n-2$ sampled data points. The co-occurrence
probability is then the fraction of samples in which $x$ and $y$ are
assigned to the same cluster.

In order to derive a measure for the accuracy of the co-occurrence
probability, we proceed as follows. The standard deviation $\sigma$ of
the co-occurrence probability is given by the binomial distribution,
\begin{equation}
\label{eq:sigma}
\sigma = \sqrt{ \frac{p \left( 1 - p\right)}{N}},
\end{equation}
where $p$ is the co-occurrence probability and $N$ is the number of
pairs sampled.  If we use a co-occurrence probability of $p=1-\alpha =
0.9$ and want to calculate this to a one standard deviation accuracy
of $\sigma=1\%$ at the $p=1-\alpha$ threshold, we need at least
\begin{displaymath}
N = \frac{\alpha\cdot (1-\alpha)}{\sigma^2} = 900
\end{displaymath}
samples where the pair of points $x$ and $y$ co-occur. Using $N =
1000$ samples hence provides good accuracy. However, as each sample
contains only one pair of points of interest, i.e, the pair $x$ and
$y$, a total of $N \cdot \binom{n}{2}$ samples is required.

The time complexity of Algorithm \algcoocpair is $\mathcal{O}(mTn^2)$,
where $m$ is the number of samples, $T$ is the time complexity of
training the clustering algorithm, i.e., finding the clustering
function $f_D$, and $n$ is the size of the dataset.

In practice, one seldom has any knowledge of the true underlying data
distribution $F$; in such case one must sample from the original
dataset. Also, the computational complexity of directly sampling from
$F$ is high, which means that this approach, is not feasible to use,
although it is a direct implementation of
Definition~\ref{def:cooc}. Instead, a natural and efficient choice for
acquiring the samples needed for calculating the co-occurrence
probabilities is to use the bootstrap approximation.

\subsubsection{Bootstrapping}
\label{sec:methods:bs}
In order to calculate the co-occurrence probabilities using the
non-parametric bootstrap approximation, $n$ data points are sampled
with replacement from the dataset $D$ and the co-occurrence
probabilities are calculated using Algorithm \algcoocbs
(Fig.~\ref{alg:cooc:bs}). The bootstrap procedure allows estimation of
the sampling distribution of a parameter, in this case the
co-occurrence probabilities, and simulates the effect of drawing new
samples from the population.

Different bootstrapping schemes can be used when calculating the
co-occurrence probabilities. For instance, the clustering function
could be constructed using out-of-bag samples. Using the bootstrap
approach should provide a good approximation of the data distribution
in most situations. The bootstrap approximation may fail if the
clustering algorithm is sensitive to duplicated data points that
necessarily appear in the bootstrap samples. One should also notice
that the above discussed method of sampling directly from $F$ is, in
fact, a variant of parametric bootstrapping, where the data generating
process $F$ is known.

The time complexity of Algorithm \algcoocbs using $m$ bootstrap
samples is $\mathcal{O}\left(m(T+n^2) \right)$, where $T$ is the time
complexity of the used clustering algorithm and $n$ is the size of the
dataset.

We will now determine the number of bootstrap samples required for a
given co-occurrence probability accuracy. The probability $p_s$ that a
randomly chosen item from a dataset of size $n$ appears in a bootstrap
sample is
 \begin{displaymath}
 p_s = 1- \left( 1- \frac{1}{n} \right)^n \text{ and } \lim_{n\rightarrow \infty} \left( p_s\right) = 1 - \frac{1}{e} .
 \end{displaymath}
Hence, the probability that a randomly chosen pair of points appear in
the bootstrap sample is therefore given by $p_s^2 \approx 0.4$ meaning
that each bootstrap sample on average covers \unit[40]{\%} of the
pairs in the dataset. The standard deviation $\sigma$ of the
co-occurrence probability when bootstrapping is now given by
Equation~\ref{eq:sigma} setting $N = n \cdot p_s^2$. Using $n = 1000$
bootstrap samples and a co-occurrence probability of $p = 0.9$ here
thus gives us a one standard deviation accuracy of \unit[1.5]{\%}.

Both of the above mentioned schemes for calculating the co-occurrence
probabilities are usable in a scenario where there is no natural way
to assign a cluster index to a previously unseen data item.  As will
be shown in the experimental evaluation, using the bootstrap agrees
well with using the true underlying distribution $F$, when $F$ is
known.

The number of samples in Algorithm~\algcoocpair needed to reach a
co-occurrence probability accuracy grows as the number of samples in
the dataset increases. In contrast, the number of samples needed for a
given accuracy using the bootstrap in Algorithm~\algcoocbs remains
constant regardless of the size of the dataset. Algorithm \algcoocbs
(non-parametric bootstrap approximation) is computationally much more
efficient than Algorithm \algcoocpair (direct estimation).

The non-parametric bootstrapping method is hence the preferred method
to be used in the calculation of the co-occurrence probabilities.

\subsubsection{Initial clustering}
The core clusters are always determined with respect to an initial
clustering (line 1 of Algorithm \algcoreclustering
(Fig.~\ref{alg:coreclustering})). By definition, the core clusters are
constructed so that the agreement between clusterings of different
samples of the data must, for the core clusters, overlap in at least
$1- \alpha$ percent of the samples. This means that the data items in
the core clusters of any given sample will overlap to $1-\alpha$
percent with any other sample. Hence, the choice of initial reference
clustering from among bootstrap samples of the dataset is in practice
arbitrary. We therefore suggest using any clustering of the full
original dataset as the reference clustering with respect to which the
core clusters are determined.

\begin{figure}[!t]
\begin{algorithm}[H]
\DontPrintSemicolon
\SetKwFunction{KwRandomize}{CoOccurrenceProbabilitiesDirect}
\SetKwInOut{Input}{input}\SetKwInOut{Output}{output}
\Input{
\begin{itemize}
\item[--] a data matrix $D$ with $n$ rows,
\item[--] a function $r$ that produces $n-2$ random data items sampled from $F$ (or its approximation),
\item[--] a clustering function $f_Z$ computed for any dataset $Z\in X^n$,
\item[--] the number of random samples $m$
\end{itemize}
       }
\Output{$P$: an $n\times n$ matrix where $P[i,j]$ gives the co-occurrence probability between data items $D[i]$ and $D[j]$}

\BlankLine
Let $A$ be an $n\times n$ matrix with all entries initialised to $0$\;
\For{$i$ {\rm in} $1$ {\rm to} $n-1$} {
\For{$j$ {\rm in} $i+1$ {\rm to} $n$} {
\For{$k$ {\rm in} $1$ {\rm to} $m$} {
Let $R$ be a vector of $D[i]$, $D[j]$, and a set of $n-2$ data items drawn from $r$\;
Add one to $A[i,j]$ and $A[j,i]$ if $D[i]$ and $D[j]$ are assigned to the same cluster in $f_R(R)$\tcc{Find a clustering function for data set $R$ and check if data items $i$ and $j$ are in the same cluster}
}
}
}
\Return{$A/m$}
\caption{CoOccurrenceProbabilitiesDirect}
\end{algorithm}
\caption{Algorithm for finding the co-occurrence probabilities of a dataset $D$.}
\label{alg:cooc:pair}
\end{figure}

\begin{figure}[!t]
\begin{algorithm}[H]
\DontPrintSemicolon
\SetKwFunction{KwRandomize}{CoOccurrenceProbabilitiesBootstrap}
\SetKwInOut{Input}{input}\SetKwInOut{Output}{output}
  \Input{
  \begin{itemize}
  \item[--]  a data matrix $D$ with $n$ rows, a clustering
  \item[--] function $f_Z$
  \item[--]the number of random samples $m$
\end{itemize}
  }
\Output{$P$: an $n\times n$ matrix where $P[i,j]$ gives the co-occurrence probability between data items $D[i]$ and $D[j]$}
\BlankLine
\tcc{$A$ and $B$ are initialised with non-zero values to insert a
  slight prior towards a flat distribution and to avoid problems with
  divisions by zero.  $A[i,j]$ counts how many times items $i$ and $j$
  occur in the same cluster and $B[i,j]$ counts how many times both
  $i$ and $j$ occur in the bootstrap sample.}
Let $A$ be an $n\times n$ matrix with all entries initialised to $1/n$\;
Let $B$ be an $n\times n$ matrix with all entries initialised to $1$\;
Let the diagonals of $A$ and $B$ be unity\;
\For{$k$ {\rm in} $1$ {\rm to} $m$}{
  Let $\overline I$ be a vector of $n$ integers sampled uniformly in random with replacement from $\{1,\ldots,n\}$ and let $Z$ be a set of $n$ data items such that $Z[i]=D[\overline I[i]]$\;
  Let $\overline c$ be a vector of cluster indices output by $f_Z(Z)$\;
  Let $\overline u$ be a vector of indices of unique values in $\overline I$\;
  \If{$|\overline u|>1$}{
    \For{$i$ {\rm in} $1$ {\rm to} $|\overline u|-1$}{
      \For{$j$ {\rm in} $i+1$ {\rm to} $|\overline u|$}{
        Let $a$ be $\overline I[\overline u[i]]$\;
        Let $b$ be $\overline I[\overline u[j]]$\;
        \tcc{increase counter as $a$ and $b$ are both present in the sample}
        Increase $B[a,b]$ and $B[b,a]$ by one\;
        \tcc{Increase counter if $a$ and $b$ are assigned to the same cluster.}
        \If{$\overline c[\overline u[i]]=\overline c[\overline u[j]]$}{
          Increase $A[a,b]$ and $A[b,a]$ by one\;
        }
      }
    }
  }
}
Let $P$ be an $n\times n$ matrix where each entry satisfies $P[i,j]=A[i,j]/B[i,j]$\ for all $i$ and $j$\;
\Return{$P$} \BlankLine
\caption{CoOccurrenceProbabilitiesBootstrap}
\end{algorithm}
\caption{Algorithm for finding the co-occurrence
  probabilities for all data items using the bootstrap approximation.}
  \label{alg:cooc:bs}
\vspace*{10mm}
\end{figure}

\subsection{Complexity of finding core clusters}
In this section we show that the problem of determining the core
clusters is \NPhard by proving the following theorem.

\begin{theorem}
Finding the core clusters is \NPhard.
\end{theorem}

\begin{proof}
We show that finding the core clusters is \NPhard by a reduction to
the {\sc clique} problem, which is a classic \NPhard problem
\citep{karp:1972:a}. The {\sc clique} problem is defined as follows:
given an undirected graph $G$ with $n\ge 2$ vertexes, find the largest
fully connected subgraph in $G$. Consider the problem of finding a
core cluster from a dataset of $n$ items with parameter $\alpha$ given
by $\alpha=1/2$ and the clustering function $f_D$ constructed as
follows. Assign all data items to the same cluster with a probability
\[
p=1/2-1/(n(n-1)),
\]
i.e., $f_D(x)=1$ for any $x$. Otherwise pick a random pair of items
$(i,j)$: if there is an edge between the items $i$ and $j$ in $G$ then
assign $i$ and $j$ to a cluster of two items, otherwise assign them to
singleton clusters; finally assign the remaining $n-2$ items to
singleton clusters. Assume that the initial clustering was (by chance)
such that all data items were assigned to the same cluster, hence,
there will be one core cluster. Now, if there is no edge in $G$
between a pair of items then the co-occurrence probability will be
\[
p=1/2-1/(n(n-1))<1-\alpha .
\]
If there is an edge in $G$ between the items $i$ and $j$ they will
co-occur in a cluster with a probability of
$p+(1-p)/(\rm{number~of~pairs})= 1/2-1/(n(n-1))+(1/2+1/(n(n-1)))\times
2/(n(n-1))\ge 1-\alpha$.

Hence, the co-occurrence probability between a pair of items is at
least $1-\alpha$ if and only if there is an edge between the
items. Therefore, the solution to the core clustering problem using
Algorithm \algcoreclustering (Fig.~\ref{alg:coreclustering}) gives
under this construction of $f_D$ the largest clique in $G$.
\end{proof}

\clearpage
\section{Experiments}
\label{experiments}

\subsection{Experimental Setup}
In the experiments we investigate the following aspects: (i) the
impact of core clustering and (ii) the differences between
Algorithm \algcoocpair (exact sampling from the known
distribution) and Algorithm \algcoocbs (bootstrap
approximation).

\subsubsection{Clustering Algorithms}
We use both unsupervised learning methods (clustering algorithms) and
supervised learning methods (classifiers) to determine the core
clustering of different datasets. All experiments are performed in R
\citep{R:2014:a}. As clustering algorithms we use one parametric
method, k-means++\citep{vassilvitskii:2007:a}, and one non-parametric
method, hierarchical clustering (hclust), as well as two robust
clustering methods based on trimming as implemented by the
\emph{tclust} R-package \citep{Fritz:2012:a}; trimmed k-means
(tkmeans\footnote{The \texttt{tclust} function with parameters
  \texttt{restr=eigen}, \texttt{restr.fact = 1} and
  \texttt{equal.weights=TRUE}}) \citep{Cuesta:1997:a} and
tclust\footnote{The \texttt{tclust} function with parameters
  \texttt{restr=eigen}, \texttt{restr.fact = 50} and
  \texttt{equal.weights=FALSE}} \citep{GarciaEscudero:2008:a}. In the
robust clustering methods based on trimming, a given fraction of the
most outlying data items are trimmed and are not part of the
clustering solution. The points that are trimmed are hence comparable
to the weak points in the core clustering method. As classifiers we
use Random Forest (RF) and support vector machines (SVM), which both
are among the best-performing classifiers, see e.g.,
\cite{delgado:2014:a}.

\subsubsection{Datasets}
As datasets in the experiments, we use the synthetic data presented
above in the motivating example, five datasets from the UCI Machine
Learning Repository (Iris, Wine Glass, Yeast and Breast Cancer
Wisconsin (BCW)) \citep{Bache:2014:a}, and the 10\% KDD Cup 1999
dataset (KDD). The properties of the datasets are described in
Table~\ref{tab:datasets}. Items with missing values in the datasets
were removed. As noted by \cite{chawla:2013:a}, the three classes
\texttt{normal}, \texttt{neptune} and \texttt{smurf} account for
98.4\% of the KDD dataset, so we selected only these three
classes. Furthermore, we performed variable selection for this
dataset, using a random forest classifier to reduce the number of
variables in the dataset to 5 (variables 5, 2, 24, 30 and 36).

\subsubsection{Experimental Procedure}
Two separate experiments were carried out. In the first experiment, we
obtain the core clustering for the synthetic and KDD dataset using
Algorithm \algcoocpair (Fig.~\ref{alg:cooc:pair}), which assumes
knowledge of the true underlying data distribution.

For the synthetic dataset, we sample directly from the data generating
distribution, a mixture of three Gaussians, which is possible since it
is known. For the KDD dataset, we initially choose a random sample of
200 data items to cluster, corresponding to 0.04\% of the size of the
dataset. Since the KDD dataset is very large in comparison to the
small sample we wish to cluster, we approximate the true distribution
of data items by randomly drawing samples from the entire KDD
dataset. The core clustering of the UCI datasets using the true
distribution is not possible as the datasets are too small.

In the second experiment, we determine the core clustering of all the
datasets in Table~\ref{tab:datasets} using the bootstrap approximation
of Algorithm \algcoocbs (Fig.~\ref{alg:cooc:bs}).

We use $1000$ bootstrap iterations for calculating the co-occurrence
probabilities and a confidence threshold of $\alpha = 0.10$ in all
experiments. It should be noted that the goal here was not to maximise
quality of clustering or classification accuracy, but to demonstrate
the effect of core clustering. All clustering and classification
functions were hence used at their default settings. The k-means++
algorithm was run ten times and the clustering solution with the
smallest within-cluster sum of squares was chosen.

To evaluate the method of core clustering, we use the external
validation metric \emph{purity}
\cite[Sec. 16.3]{manning:2009:a}. Purity ranges from zero to one, with
unity denoting a perfect match to the ground truth class structure.

All source code used for the experiments, including an R package
\verb|corecluster| implementing the core clustering algorithm, is
available for
download\footnote{\url{https://github.com/bwrc/corecluster-r/}}.

\begin{table}[!ht]
\renewcommand{\arraystretch}{1.3}
\caption{Properties of the datasets used in the experiments. The
  figure in parentheses for the KDD dataset is the total size of the
  dataset, from which 200 instances where chosen as the dataset to
  cluster. Dataset size is the number of instances after removal of
  items with missing values.}
\label{tab:datasets}
\centering
\begin{tabular}{lrccc}
 \toprule
 \textbf{dataset} & \textbf{Size} & \textbf{Classes} & \textbf{Attributes} & \textbf{Major class}\\
 \midrule
synthetic                     & 150          & 3  & 2  & 0.33\\
iris                          & 150          & 3  & 4  & 0.33 \\ 
wine                          & 178          & 3  & 13 & 0.40 \\ 
glass                         & 214          & 6  & 9  & 0.36 \\ 
BCW                           & 683          & 2  & 9  & 0.65 \\ 
yeast                         & 1484         & 10 & 8  & 0.31 \\ 
KDD                           & 200 (485269) & 3  & 5  & 0.58 \\ 
\bottomrule\\
\end{tabular}
\end{table}

\subsection{Experimental Results}
\subsubsection{Agreement Between True Distribution and Bootstrap Approximation}
The agreement between core clustering using the true distribution
(Algorithm \algcoocpair) and the bootstrap approximation (Algorithm
\algcoocbs) is shown in Table~\ref{tab:res:confusion}. The table shows
the confusion matrices for the used clustering algorithms when run on
the synthetic and KDD datasets, using Algorithm \algcoocpair and
Algorithm \algcoocbs. The confusion matrices show that most of the
points fall in either the top left or bottom right corner. The results
from the two algorithms agree if the majority of points is located on
the diagonal. For all clustering algorithms and both datasets, with
the exception of hierarchical clustering and trimmed k-means for the
synthetic dataset, only a minor proportion of data items are in
discord between using the true distribution and using the bootstrap
approximation.

\subsubsection{Core Clustering for Datasets}
The experimental results for all datasets using all classifiers are
shown in Table~\ref{tab:res:true} and Table~\ref{tab:res:bs}. The
quality of the clustering is assessed with purity using the known
class labels of the datasets. The table shows the purity using the
original clustering ($\textrm{P}_\textrm{o}$), and using core
clustering ($\textrm{P}_\textrm{c}$). The fraction of weak points
($w$) is also provided. The weak points are ignored when calculating
purity for core clustering. The robust clustering algorithms trimmed
k-means and tclust were set to trim 5\% of the points. In some cases,
the tclust, tkmeans and random forest algorithms failed to cluster a
particular data sample. For Algorithm \algcoocpair and Algorithm
\algcoocbs, a new sample was then obtained, and if a clustering
solution was not obtained in 5 iterations, this iteration was
discarded.

The results in Table~\ref{tab:res:bs} systematically show that core
clustering improves purity in all cases, with the exception of
k-means++ clustering for the KDD dataset, for which the drop in purity
is 0.01.

It can also be seen that the results calculated using the true
distribution shown in Table~\ref{tab:res:true} agree with those
calculated using the bootstrap approximation in
Table~\ref{tab:res:bs}.

\subsubsection{Visualisations of Core Clusterings}
Examples of core clusterings obtained using the bootstrap
approximation of the synthetic, iris, and BCW datasets using different
clustering algorithms are shown in Figure~\ref{fig:res:all}. Trimmed
points, for tclust and tkmeans, are marked with stars.

Figure~\ref{fig:res:hclust:synthetic}, showing core clustering of the
synthetic dataset using hierarchical clustering, can be compared to
the core clustering of this dataset using k-means++ shown in
Figure~\ref{fig:ex:coreclusters}, also calculated using Algorithm
\algcoocbs. It is clear that a large number of points are weak points
when using hierarchical clustering (53\% of the points, as also seen
in Table~\ref{tab:res:bs}). The same applies to core clustering using
tclust (Figure~\ref{fig:res:tclust:synthetic}), which categorises 85\%
of the points as weak points. The interpretation is that hierarchical
clustering and tclust exhibit a high variability in the clustering
outcome on different iterations, which leads to core clusters with
small radii. This means, that these algorithms are not well suited for
the clustering of this dataset. The core clustering using SVM
(Figure~\ref{fig:res:tclust:synthetic}) only categorises 20\% of the
points as weak points, comparable to the results for k-means++.

The core clustering of the iris dataset using k-means
(Figure~\ref{fig:res:kmeans:iris}) clearly shows that the weak points
are located between the two rightmost clusters in the figure. This is
also visible when using hierarchical clustering and trimmed
k-means. For trimmed k-means some peripheral points have also been
excluded.

The core clusterings for the BCW dataset shown in
Figures~\ref{fig:res:kmeans:bcw}, \ref{fig:res:hclust:bcw} and
\ref{fig:res:tclust:bcw} vary depending on the clustering
algorithm. It is clear, that the k-means++ and tclust algorithms are
better suited for clustering this dataset than hierarchical
clustering, which must discard 22\% of the data items as weak points.

\begin{figure*}
\centering
\subfloat[Synthetic -- SVM]{\includegraphics[width = 0.33 \textwidth]{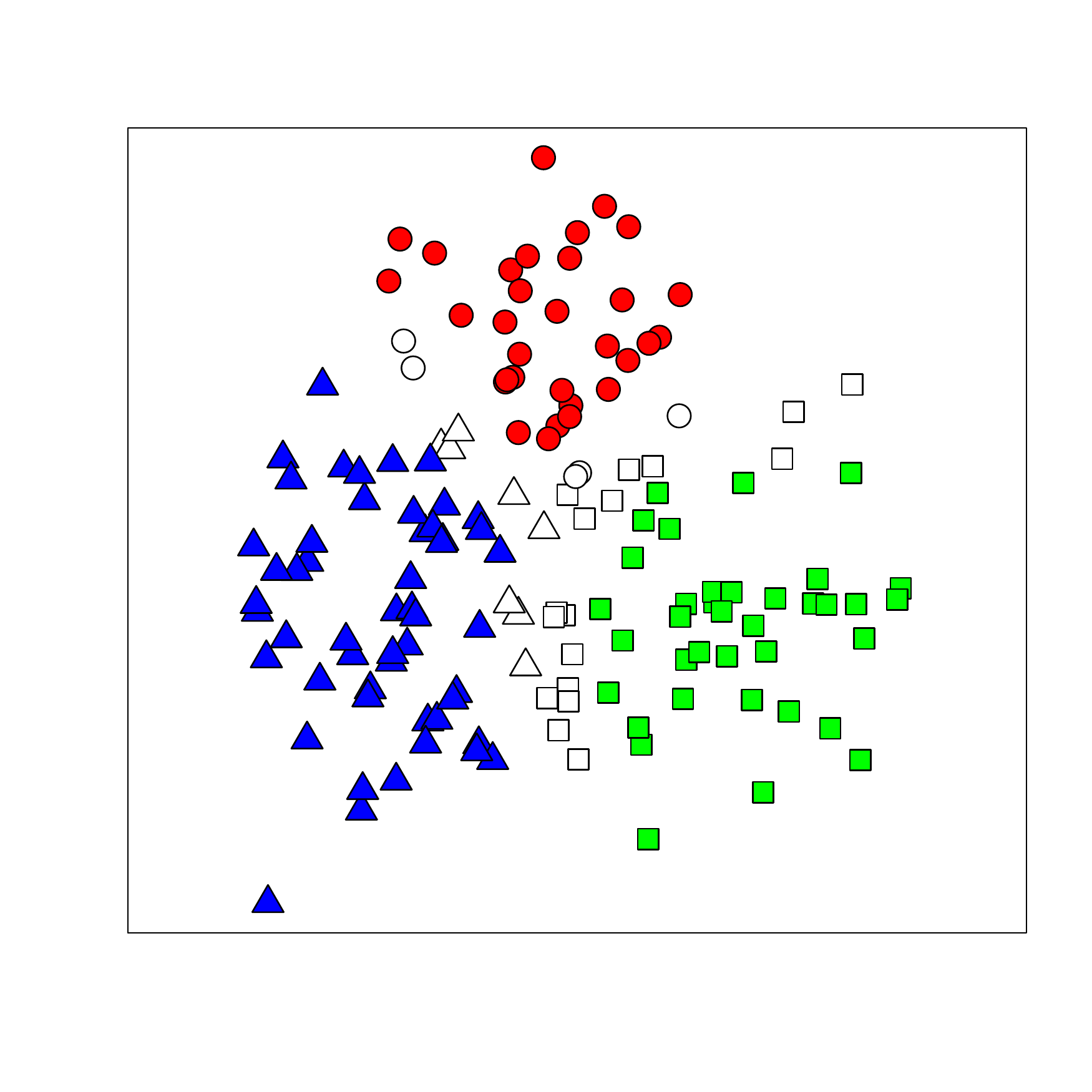} \label{fig:res:svm:synthetic}}
\subfloat[Synthetic -- hclust]{\includegraphics[width = 0.33 \textwidth]{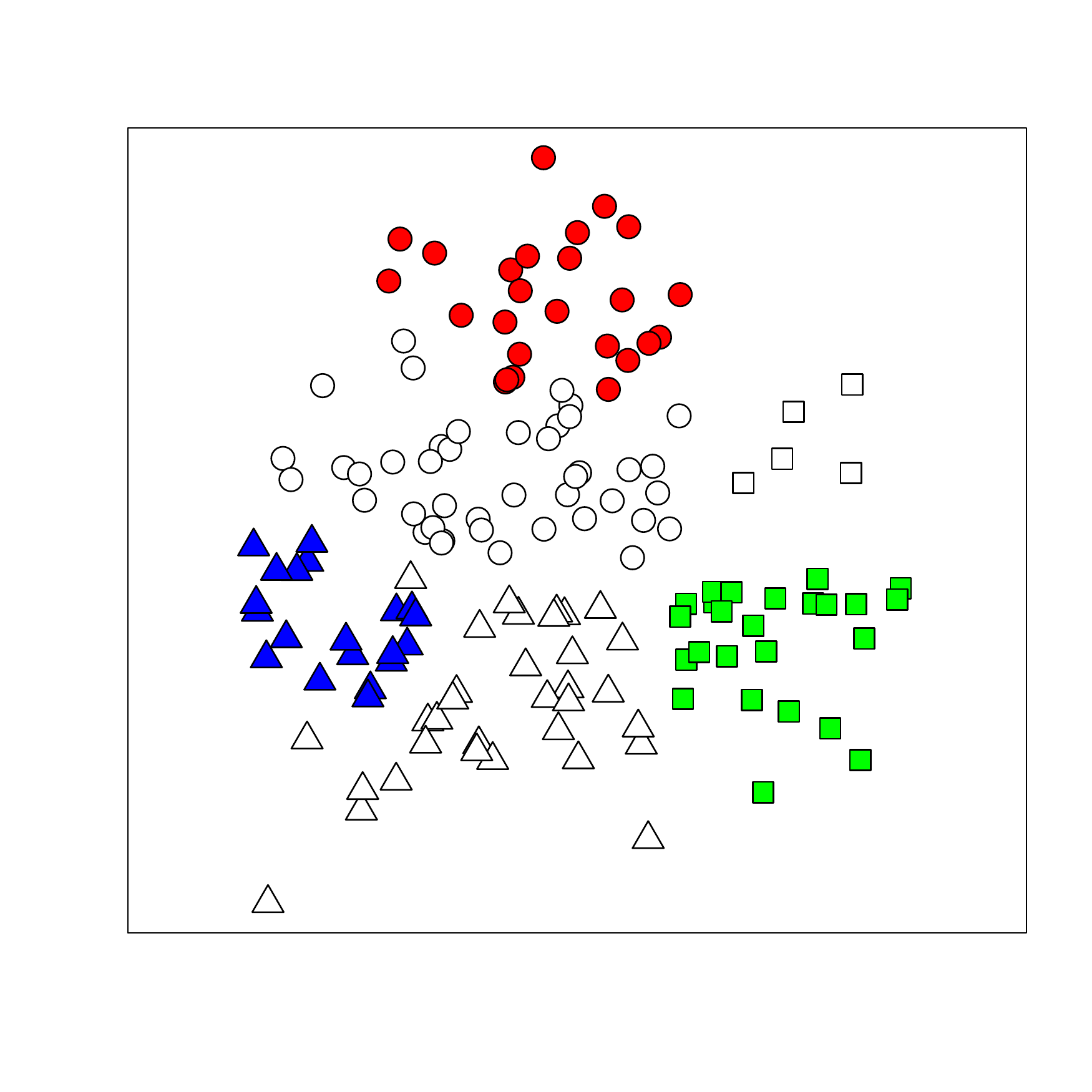} \label{fig:res:hclust:synthetic}}
\subfloat[Synthetic -- tclust]{\includegraphics[width = 0.33 \textwidth]{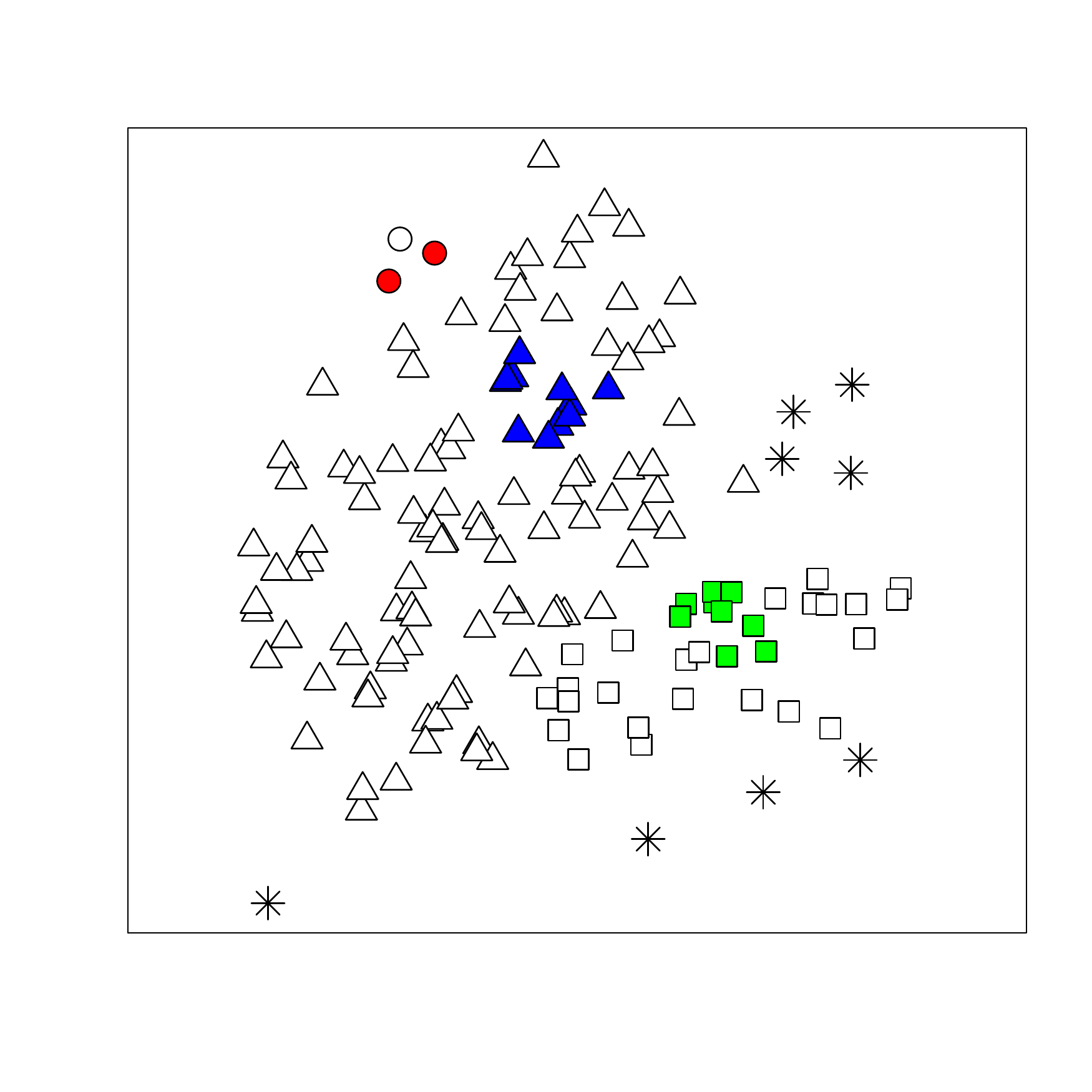} \label{fig:res:tclust:synthetic}}

\subfloat[Iris -- k-means++]{\includegraphics[width = 0.33 \textwidth]{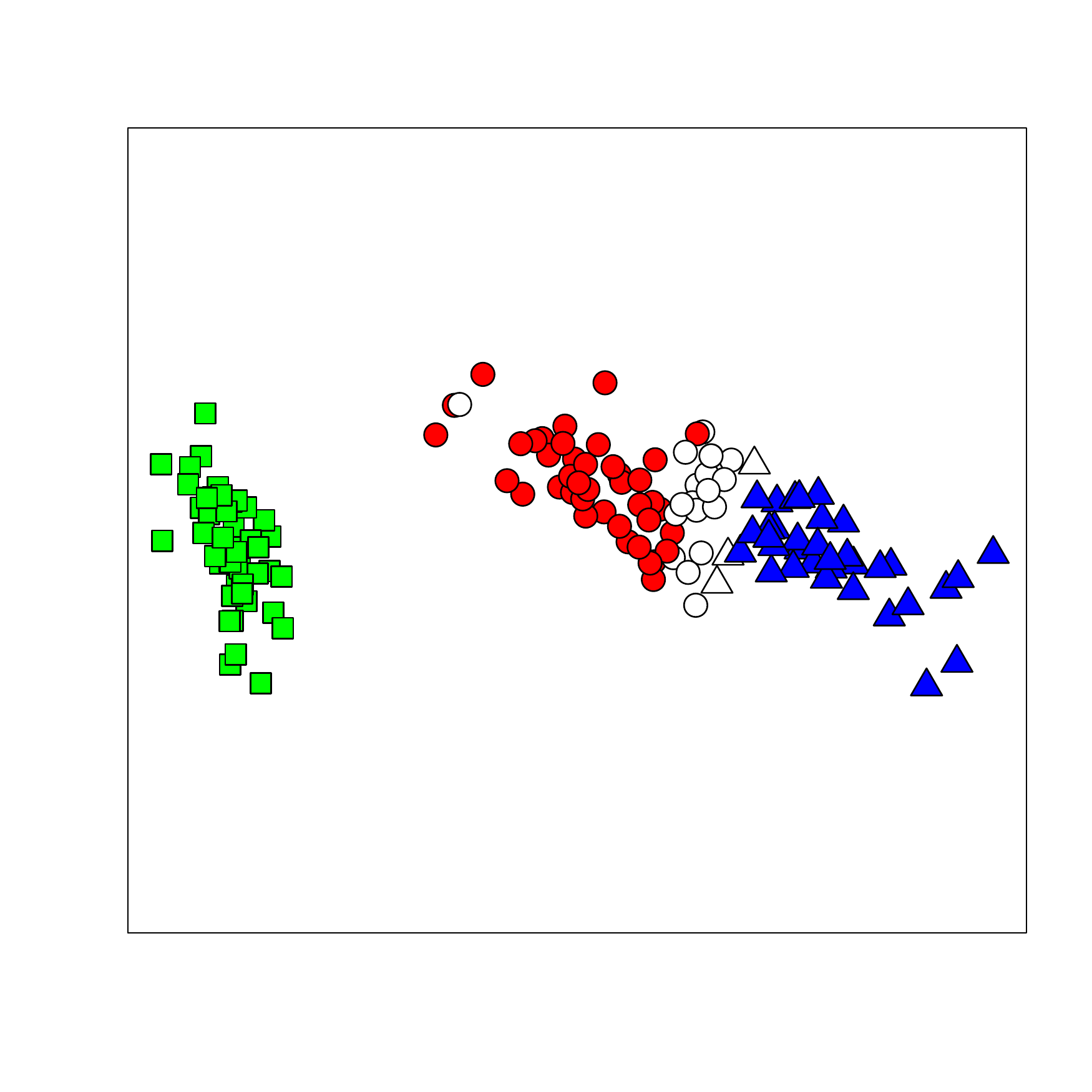} \label{fig:res:kmeans:iris}}
\subfloat[Iris -- hclust]{\includegraphics[width = 0.33 \textwidth]{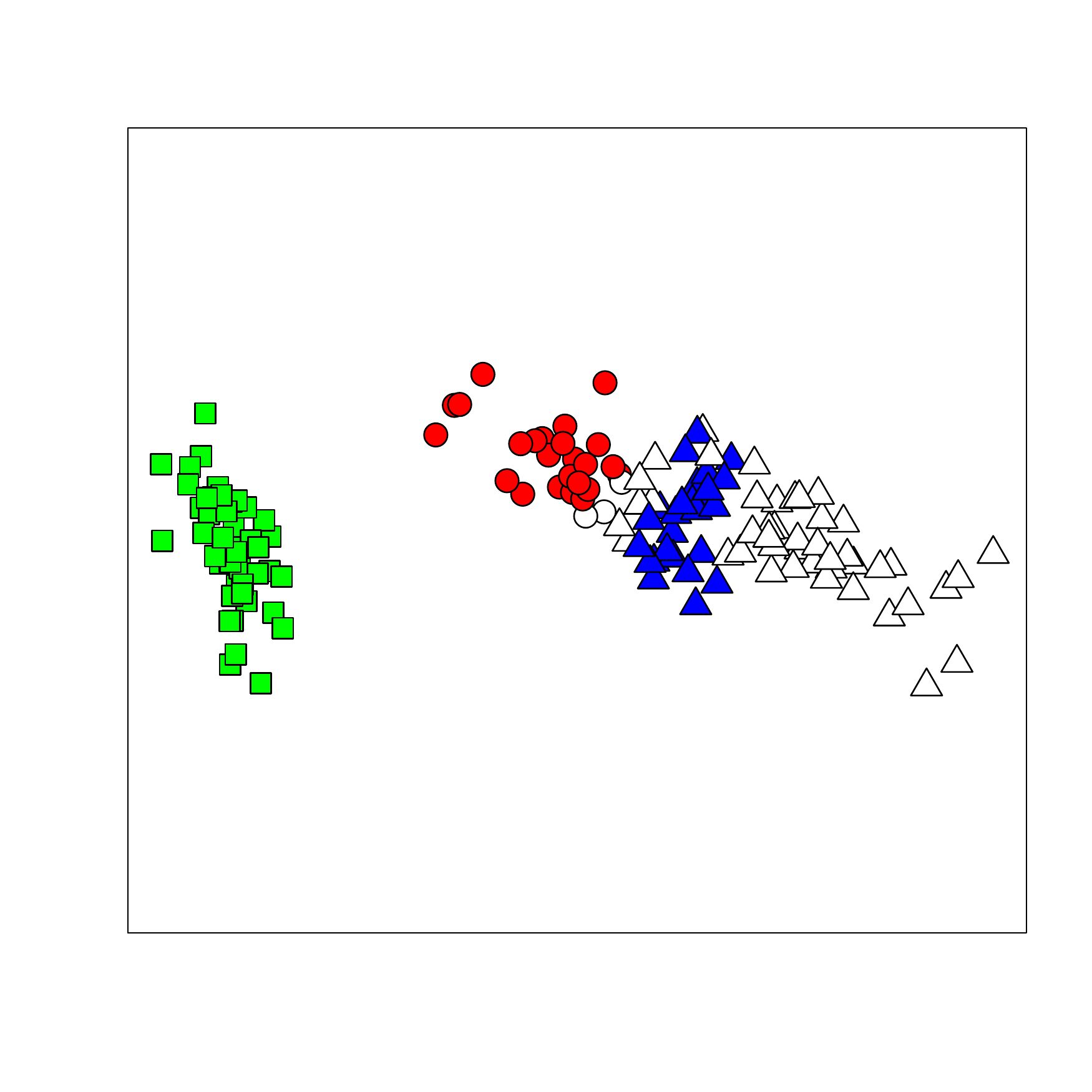} \label{fig:res:hclust:iris}}
\subfloat[Iris -- tkmeans]{\includegraphics[width = 0.33 \textwidth]{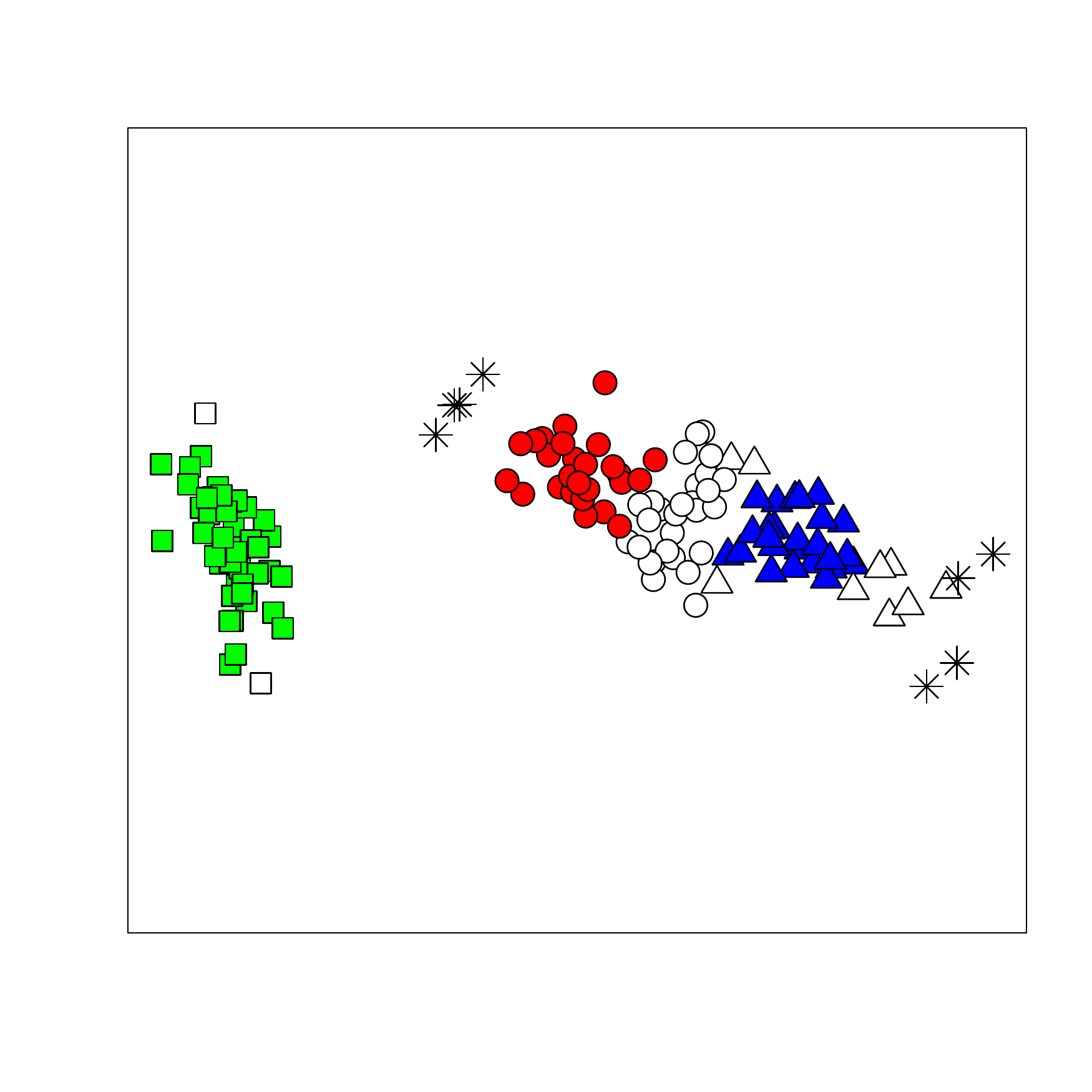} \label{fig:res:tkmeans:iris}}

\subfloat[BCW -- k-means++]{\includegraphics[width = 0.33 \textwidth]{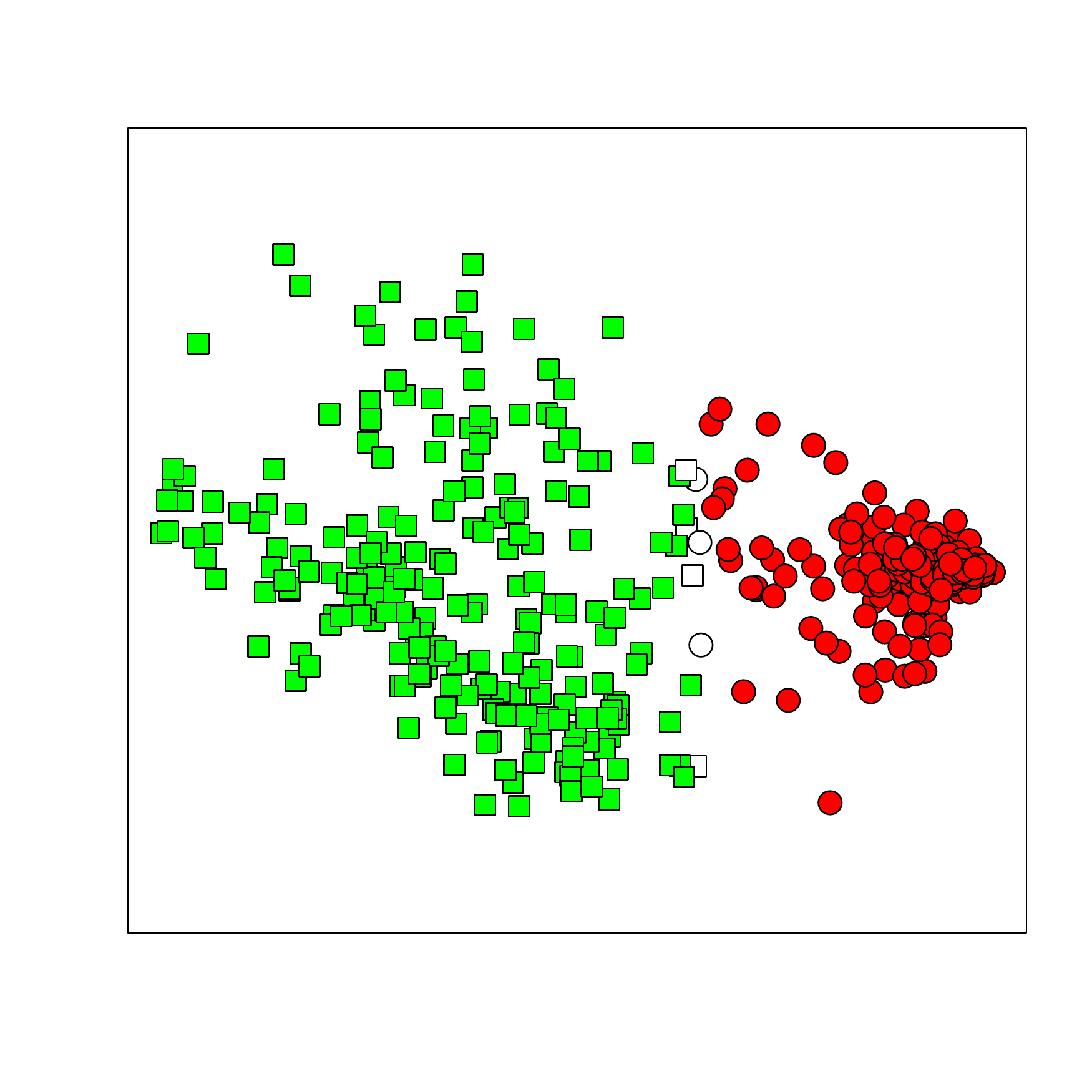} \label{fig:res:kmeans:bcw}}
\subfloat[BCW -- hclust]{\includegraphics[width = 0.33 \textwidth]{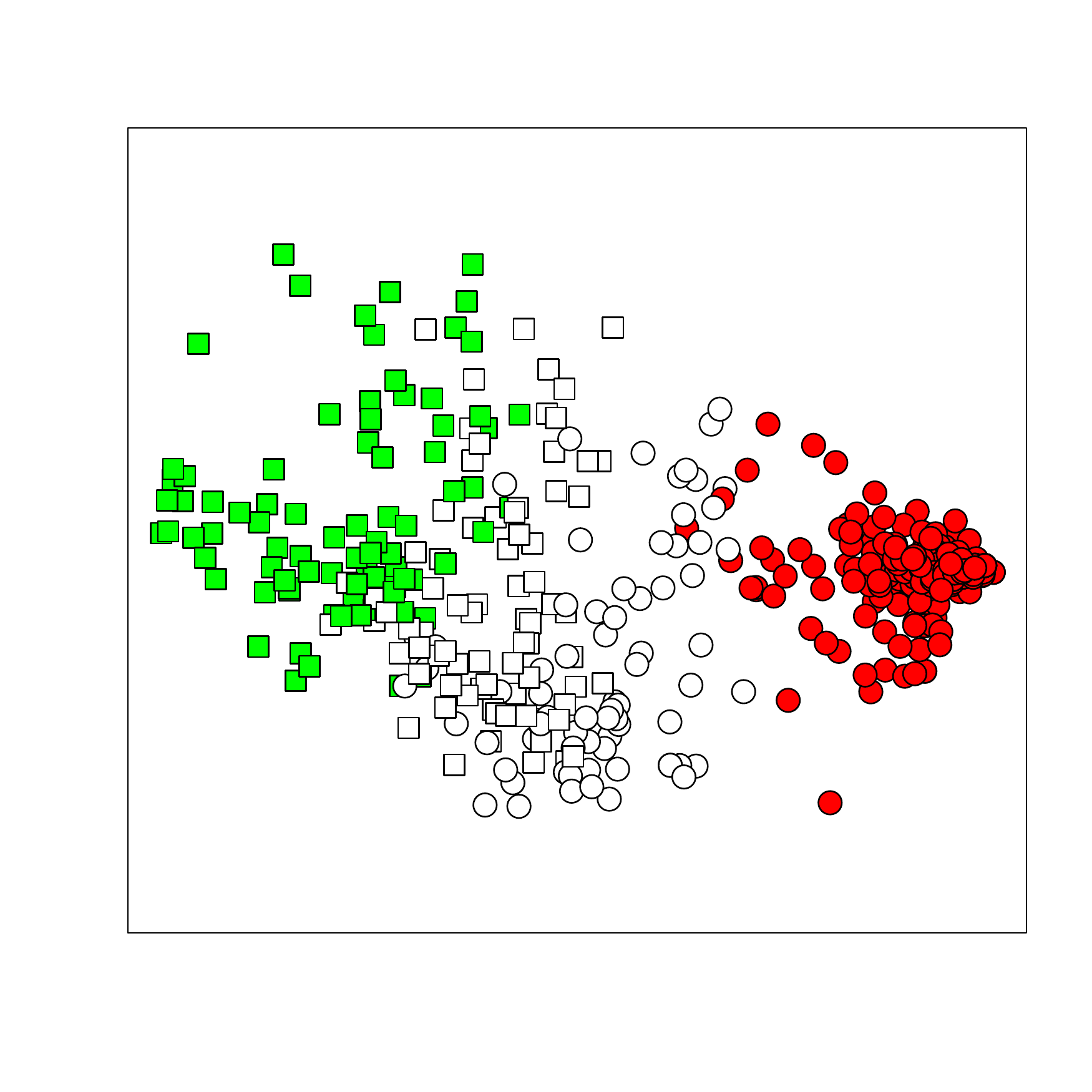} \label{fig:res:hclust:bcw}}
\subfloat[BCW -- tclust]{\includegraphics[width = 0.33 \textwidth]{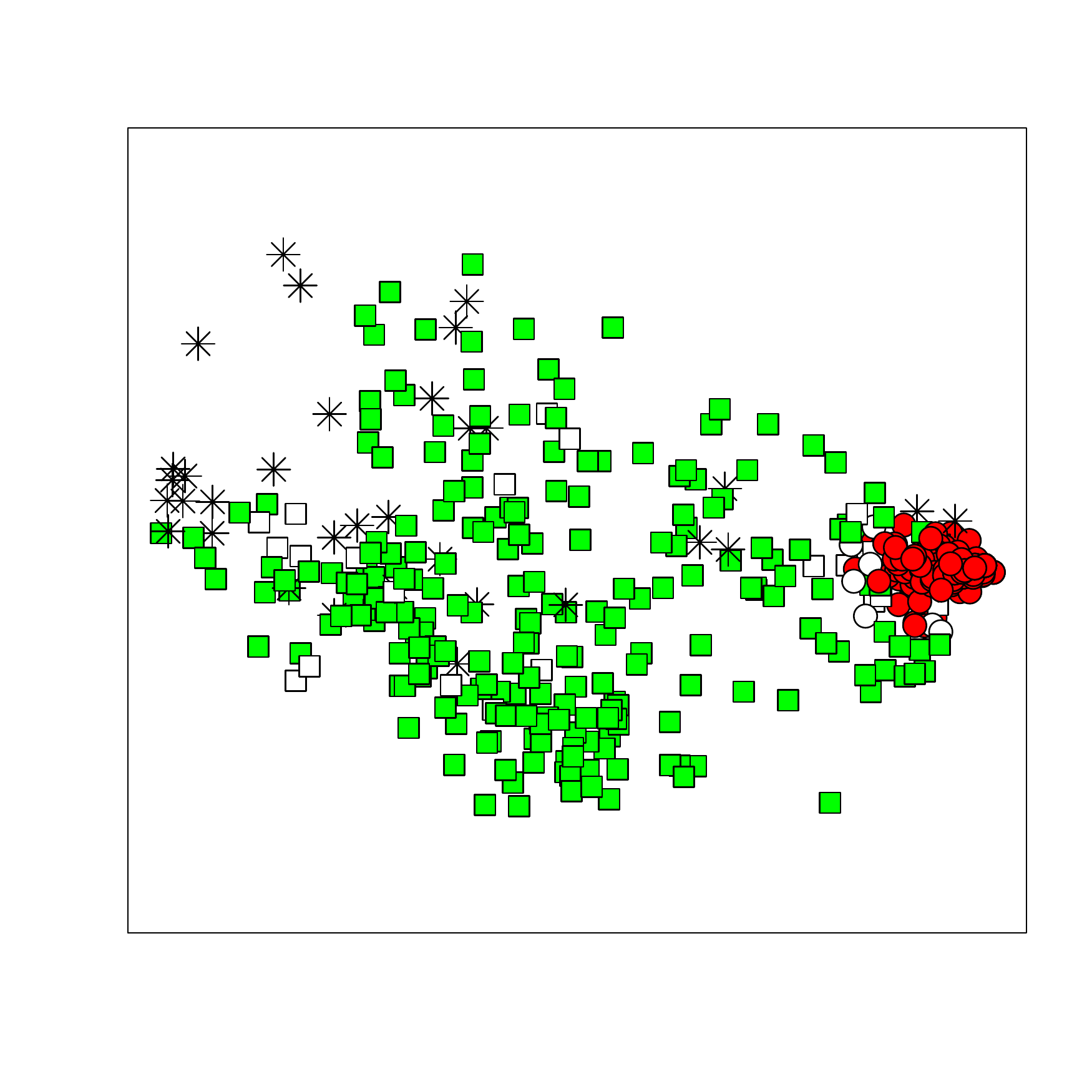} \label{fig:res:tclust:bcw}}
\caption{Clusterings of the synthetic, iris and
  breast-cancer-wisconsin (BCW) datasets. For iris and BCW the two
  first principal components are plotted. The original clustering is
  shown using different symbols for each cluster. Core clusters are
  shown with filled symbols and weak points are unfilled. Trimmed
  points (for the tclust and tkmeans algorithms) are marked with
  stars.}
\label{fig:res:all}
\end{figure*}


\begin{table}[!ht]
\caption{Agreement between core clusterings using the true
  distribution (Algorithm \algcoocpair) and the bootstrap
  approximation (Algorithm \algcoocbs). The results are presented as a
  confusion matrix, described to the right of the table. The top left
  (denoted by $a$) gives the number of data items categorised as
  matching core points by both the true distribution and the the
  bootstrap approximation, $d$ gives the data items categorised as
  weak points by both algorithms, whereas $b$ and $c$ give the number
  of points where the two methods disagree. The number of points for
  tclust and tkmeans does not add up to the total size of the dataset,
  since 5\% of the points are trimmed.}
\label{tab:res:confusion}
\centering

\newcolumntype{C}{>{\centering\arraybackslash}p{2em}}

\newcommand{\cfmat}[4]{
$ \begin{array}{C|C}
#1 & #2 \\
\hline
#3 & #4 
\end{array} $
}

\begin{minipage}{\linewidth}
\centering
\begin{tabular}{lcc}
\toprule
  \textbf{algorithm} & \textbf{synthetic dataset} & \textbf{KDD dataset}\\
\midrule
hclust & \cfmat{40}{4}{30}{76} & \cfmat{198}{0}{0}{2}\\ 
\midrule 
k-means++ & \cfmat{121}{3}{3}{23} & \cfmat{181}{0}{0}{19}\\ 
\midrule 
random forest & \cfmat{150}{0}{0}{0} & \cfmat{200}{0}{0}{0}\\ 
\midrule 
SVM & \cfmat{116}{6}{4}{24} & \cfmat{200}{0}{0}{0}\\ 
\midrule 
tclust & \cfmat{10}{11}{12}{109} & \cfmat{185}{0}{0}{5}\\ 
\midrule 
tkmeans & \cfmat{90}{20}{1}{31} & \cfmat{181}{3}{0}{6}\\ 
\bottomrule 
\end{tabular}
\end{minipage}\\
\vspace*{5pt}
\begin{minipage}{\linewidth}
  \centering
 $\begin{array}{ccc|c}
&&\multicolumn{2}{c}{\textbf{bootstrap}}\\
&& $\textrm{core}$ & $\textrm{weak}$ \\
\multirow{2}{*}{\rotatebox{90}{\textbf{true}}}
& $\textrm{core}$ & a & b \\
\cline{3-4}
  &$\textrm{weak}$ & c & d 
 \end{array} $ 
\end{minipage}
\end{table}

\begin{table}[!ht]
\caption{Results for core clustering of the synthetic and KDD datasets
  using the true distribution (Algorithm \algcoocpair). Here
  $\textrm{P}$ stands for purity with the subscripts $o$ and $c$
  denoting original clustering and core clustering, respectively. The
  fraction of weak points is denoted by $w$.}
\label{tab:res:true}
\centering
\begin{tabular}{ll ccc ccc}
 \toprule
 \textbf{dataset} & \textbf{algorithm} & $\textrm{P}_\textrm{o}$ &  $\textrm{P}_\textrm{c}$ & $w$ \\
 \midrule
\multirow{6}{*}{synthetic} 
   & hclust & 0.71 & 1.00 & 0.71\\ 
   & k-means++ & 0.83 & 0.90 & 0.17\\ 
   & random forest & 1.00 & 1.00 & 0.00\\ 
   & SVM & 0.85 & 0.92 & 0.19\\ 
   & tclust & 0.61 & 1.00 & 0.86\\ 
   & tkmeans & 0.83 & 0.89 & 0.27\\ 
\midrule
\multirow{6}{*}{KDD} 
   & hclust & 0.82 & 0.83 & 0.01\\ 
   & k-means++ & 0.82 & 0.81 & 0.10\\ 
   & random forest & 1.00 & 1.00 & 0.00\\ 
   & SVM & 1.00 & 1.00 & 0.00\\ 
   & tclust & 0.86 & 0.88 & 0.08\\ 
   & tkmeans & 0.86 & 0.89 & 0.08\\ 
\bottomrule
\end{tabular}
\end{table}

 
\begin{table}[!ht]
\caption{Results for core clustering of the synthetic, UCI datasets
  and KDD datasets using the bootstrap approximation
  (Algorithm \algcoocbs). The columns in the table are the same
  as in Table~\ref{tab:res:true}.}
\label{tab:res:bs}
\centering
\begin{tabular}{ll ccc ccc}
 \toprule
 \textbf{dataset} & \textbf{algorithm} & $\textrm{P}_\textrm{o}$ &  $\textrm{P}_\textrm{c}$ & $w$ \\
 \midrule
\multirow{6}{*}{BCW} 
   & hclust & 0.89 & 0.98 & 0.22\\ 
   & k-means++ & 0.96 & 0.97 & 0.01\\ 
   & random forest & 1.00 & 1.00 & 0.00\\ 
   & SVM & 0.98 & 0.99 & 0.02\\ 
   & tclust & 0.92 & 0.93 & 0.10\\ 
   & tkmeans & 0.96 & 0.97 & 0.08\\ 
\midrule
\multirow{6}{*}{glass} 
   & hclust & 0.50 & 0.57 & 0.16\\ 
   & k-means++ & 0.59 & 0.60 & 0.17\\ 
   & random forest & 1.00 & 1.00 & 0.00\\ 
   & SVM & 0.79 & 0.91 & 0.26\\ 
   & tclust & 0.57 & 0.76 & 0.63\\ 
   & tkmeans & 0.59 & 0.71 & 0.51\\ 
\midrule
\multirow{6}{*}{iris} 
   & hclust & 0.84 & 0.88 & 0.32\\ 
   & k-means++ & 0.89 & 0.98 & 0.15\\ 
   & random forest & 1.00 & 1.00 & 0.00\\ 
   & SVM & 0.97 & 0.99 & 0.05\\ 
   & tclust & 0.98 & 1.00 & 0.47\\ 
   & tkmeans & 0.89 & 0.98 & 0.33\\ 
\midrule
\multirow{6}{*}{synthetic} 
   & hclust & 0.71 & 1.00 & 0.53\\ 
   & k-means++ & 0.83 & 0.90 & 0.17\\ 
   & random forest & 1.00 & 1.00 & 0.00\\ 
   & SVM & 0.85 & 0.95 & 0.20\\ 
   & tclust & 0.61 & 1.00 & 0.85\\ 
   & tkmeans & 0.83 & 0.90 & 0.39\\ 
\midrule
\multirow{6}{*}{wine} 
   & hclust & 0.67 & 0.75 & 0.50\\ 
   & k-means++ & 0.70 & 0.74 & 0.34\\ 
   & random forest & 1.00 & 1.00 & 0.00\\ 
   & SVM & 1.00 & 1.00 & 0.01\\ 
   & tclust & 0.70 & 0.72 & 0.16\\ 
   & tkmeans & 0.70 & 0.71 & 0.18\\ 
\midrule
\multirow{6}{*}{yeast} 
   & hclust & 0.39 & 0.70 & 0.91\\ 
   & k-means++ & 0.48 & 0.59 & 0.77\\ 
   & random forest & 0.99 & 1.00 & 0.04\\ 
   & SVM & 0.64 & 0.76 & 0.33\\ 
   & tclust & 0.52 & 0.65 & 0.89\\ 
   & tkmeans & 0.53 & 0.65 & 0.67\\ 
\midrule
\multirow{6}{*}{KDD} 
   & hclust & 0.82 & 0.83 & 0.01\\ 
   & k-means++ & 0.82 & 0.81 & 0.10\\ 
   & random forest & 1.00 & 1.00 & 0.00\\ 
   & SVM & 1.00 & 1.00 & 0.00\\ 
   & tclust & 0.86 & 0.88 & 0.08\\ 
   & tkmeans & 0.86 & 0.91 & 0.10\\ 
 \bottomrule\\
\end{tabular}
\end{table}

\begin{table}[!ht]
\caption{Running times in seconds for core clustering of the datasets, using 1000 bootstrap iterations.}
\label{tab:res:runningtimes}
\centering
\begin{tabular}{lcccccc}
  \toprule
 & \multicolumn{6}{c}{\textbf{algorithm}}\\
 \cmidrule{2-7}
 \textbf{dataset} & hclust & k-means++ & RF & SVM & tclust & tkmeans \\ 
  \midrule
synthetic & 1 & 5 & 37 & 5 & 17 & 19 \\ 
  KDD     & 2 & 6 & 44 & 6 & 26 & 18 \\ 
  iris    & 1 & 6 & 38 & 6 & 30 & 23 \\ 
  wine    & 3 & 8 & 71 & 10 & 112 & 103 \\ 
  glass   & 3 & 12 & 104 & 14 & 162 & 173 \\ 
  BCW     & 32 & 13 & 228 & 23 & 179 & 103 \\ 
  yeast   & 156 & 91 & 952 & 411 & 2134 & 1694 \\ 
   \bottomrule
\end{tabular}
\end{table}

\subsection{Scalability}
As shown above, the method of finding core clusters is at least as
hard as finding the maximum clique, but the search for core clusters
is easily parallelisable in terms of the original clusters, i.e., each
of the core clusters can be found independently in parallel within
each of the original clusters (lines 6--7 in Algorithm
\algcoreclustering). Changing the number of replicates used in the
bootstrap also affect the running time; the time complexity is
quadratic with respect to the number of items in the dataset, as shown
in Sections~\ref{sec:methods:F} and \ref{sec:methods:bs}. Typical
running times on the datasets used in the experiments for this paper
with the recommended bootstrap method (Algorithm \algcoocbs) are
presented in Table~\ref{tab:res:runningtimes}, using R-code with some
C++ on a 1.8 GHz Intel Core i7 CPU and 1000 bootstrap iterations. The
baseline Algorithm \algcoocpair, which has been presented for
comparison, is substantially slower, on the order of days. Clearly,
there is a large variation in the running times between the different
clustering functions, due to the underlying implementation of the
algorithms. The running time of core clustering is dominated by the
time required by the used clustering algorithm.

\section{Discussion}
\label{discussion}
The experimental results show that core clustering improves the
homogeneity of the clusters, i.e., in-cluster consistency increases as
weak points are excluded from the core clusters.  The agreement
between use of the true underlying data distribution and the bootstrap
approximation is good. This means that the bootstrap method of
calculating the co-occurrence probabilities offers a computationally
efficient and feasible way of determining the core clusters.

Core clustering reflects the interaction between the clustering
algorithm and the data. For some clustering algorithms the result is
not deterministic, which means that factors such as, e.g., choice of
initial cluster centres affects the clustering outcome on different
runs. However, this can be overcome by using methods that optimise the
initial conditions, such as the k-means++ method used here, or by
combining the output from multiple runs on the same dataset as done,
e.g., in \citet{gionis:2007:a}.

Unstable clustering functions, as exemplified in the results using
hierarchical clustering for the synthetic dataset, produce core
clusters with small radii and a large number of weak points. In
general, the cluster radii are proportional to the value of $\alpha$,
i.e., a low $\alpha$ means that the co-occurrence probability
$1-\alpha$ will be high which in turn means that the core clusters
will be small due to this strict criterion. Conversely, a high
$\alpha$ means that the co-occurrence probability is low, and hence
the core clusters will be large. However, it should also be noted that
the size of the core clusters depends on the characteristics of the
data and on the assumptions of the clustering function.

The core clusters make it possible to detect an unstable clustering
algorithm and find the data items for which the cluster membership is
uncertain. The benefit of using core clustering is in the statistical
guarantee it gives on the co-occurrence of data items in the same
cluster, and this helps in the interpretation of the structure of the
dataset. Core clustering can be used to gain insight into how the
clustering algorithm is using the data, by investigating the size of
the core clusters, e.g., using different parameters for the clustering
algorithm.

The number of weak points detected in the core clustering of a dataset
can be used as the instability measure of a clustering algorithm,
which allows the stability of a clustering algorithm in terms of the
correct number of clusters in the data to be investigated, see, e.g.,
\cite{Luxburg:2010:a}. In practice, this means that core clusterings
of a dataset using a varying number of clusters each yield a different
number of weak points. The core clustering with the fewest number of
weak points is the most stable clustering, indicating how many
clusters the data might contain.

The method of core clustering can also be viewed in terms of
hypothesis testing; the points within the core clusters are guaranteed
to co-occur with a given probability. The core clusters hence allows
the testing of the hypothesis whether two data items belong to the
same cluster, at the desired confidence level. This has implications
for using clustering algorithms in explorative data mining tasks. The
statistical guarantee given by the core clusters can be particularly
useful in certain application areas, e.g., in the medical domain,
since core clustering can be used to explore the structure of a
dataset with a confidence guarantee on the clusters. Clustering is
also used in bioinformatics to identify groups of genes with similar
expression patterns. Core clusters could be valuable also in this
context.

The method of core clustering is in itself model-free, making no
assumptions regarding the structure of the data. The only assumptions
are those imposed by the used clustering algorithm. Core clustering
can be used in conjunction with any unsupervised or supervised
learning algorithm, as shown in this paper. A useful property is that
core clustering can be used to make a non-robust clustering algorithm,
such as traditional k-means, more robust, as shown in the experiments
above.

Existing robust clustering methods try to detect outlying points in
the dataset using, e.g., some distance metric. These methods also
typically require that the proportion of points to be discarded must
be given in advance. In contrast, in core clustering one only needs to
specify the confidence level for the co-occurrences of the items in
the core clusters. This inclusive criterion can be viewed as being
more natural than specifying what fraction of data items to discard.

Usually, outliers in a dataset consist of data items in the periphery
of clusters. As noted by \cite{Fritz:2012:a}, it is also important to
remove ``bridge points'', i.e., data items located between
clusters. Using the core clustering method, it is precisely the weak
points not in the core clusters that are the bridge points. Hence,
core clustering can be used to augment any clustering algorithm to
make it capable of removing bridge points, without making any model
assumptions on the data. Core clustering can be used in conjunction
with robust clustering functions, as shown in this paper. In this
case, both peripheral outliers and bridge points can be efficiently
detected.

In the domain of classifiers, the core clusters can be interpreted as
presenting the set of data points for which the classifier output is
consistent. Core clustering could be used jointly with the
sampling-based method introduced in \citet{Henelius14DAMI} allowing
introspection into the way a classifier uses the features of the data
when making predictions. This would make it possible to study the
interplay between the features in the data in areas where the
classification results are robust and in the more uncertain areas
represented by the weak points.

The problem of core clustering can also be viewed from a Bayesian
perspective. Assume that the dataset $D$ obeys a Bayesian mixture
model of $k$ components \cite{gelman:2004:a}. In this case, the
posterior co-occurrence probability of two data items $i,j \in D$
belonging to the same mixture component can be computed using the
standard Bayesian machinery.

As noted by \cite[Sec. 8.4]{hastie:2009:a}, the bootstrap represents
an approximate nonparametric, noninformative posterior distribution
for a parameter of interest. Hence, the non-parametric bootstrap
approximates the Bayesian co-occurrence probability, if the clustering
function gives the maximum likelihood mixture assignment for the
Bayesian mixture model.  For such a clustering function core clusters
can therefore be interpreted as sets of points for which the posterior
probability of two items occurring in the same cluster is at least
$1-\alpha$. Obviously, the core clusters depend on the modelling
assumptions: if the mixture model fits the data well, the
co-occurrence probabilities tend to be higher and the core clusters
larger than if the model fits the data poorly.

The advantage of the bootstrap approach over direct Bayesian treatment
is that we do not need to know the underlying model, or even if the
clustering function gives the maximum likelihood estimate of any
Bayesian model. The close relation of the bootstrap method to the
Bayesian way of computing co-occurrence probabilities gives additional
insight into interpreting the results and the interplay between the
data and the modelling assumptions here implicit in the clustering
function.

The core clustering method is versatile, yet conceptually simple. The
sampling of data items can be performed in many different
ways. However, since the nature of the true underlying data
distribution is seldom known in real applications, the non-parametric
bootstrapping of data items for the calculation of the co-occurrences
used here is the most feasible approach and it is also computationally
fast.

\clearpage
\section{Concluding remarks}
\label{cr}
This paper presents a conceptually simple and efficient method for
finding statistically robust clusters. The method is independent of
the used clustering algorithm and is equally usable with both
clustering algorithms and classifiers.

As demonstrated in the experiments in this paper, the agreement
between the bootstrap approximation and the true distribution is
high. Furthermore, different bootstrap schemes can be devised using,
e.g., out-of-bag samples for cluster estimation. The bootstrap
approximation is usable if the clustering algorithm is not sensitive
to the fact that the bootstrap approximation necessarily produces
duplicated data points. It would also be possible to mitigate this
issue, e.g., by jittering the bootstrapped data, as noted by
\cite{henning:2007:a}. Another approach would be to estimate the
distribution of data items parametrically and use the obtained
distribution to generate new data sets, e.g., using a parametric
bootstrap approach. It is therefore possible to fine-tune the way the
co-occurrences are calculated, if needed.

The core clusters are naturally interpreted as the strong points in
the original clusters obtained using some clustering algorithm, which
makes the interpretation of core clusters straightforward. Core
clustering can hence be used to make a non-robust clustering
algorithm, such as the traditional k-means or hierarchical clustering,
more robust by providing a probabilistic guarantee for the cluster
membership of data items co-occurring in the same core cluster.

Core clustering further extends robust clustering algorithms by
allowing points lying on the border between clusters to be excluded,
as shown in the experimental results in this paper, without model
assumptions or use of distance metrics.

Summarising, core clustering can be used to find statistically valid
core clusters using any clustering or classification algorithm and any
dataset which can be resampled. The method is generic and no
additional assumptions, such as distance or similarity measures, are
needed.

\section*{Acknowledgements}
This work was supported by the Academy of Finland (decision 288814),
Tekes (Revolution of Knowledge Work project), and the High-Performance
Data Mining for Drug Effect Detection project at Stockholm University,
funded by the Swedish Foundation for Strategic Research under grant
IIS11-0053.

\bibliographystyle{abbrvnat_mod}
\bibliography{clustering_with_confidence}

\end{document}